\documentclass[final,12pt]{colt2019} 

\usepackage{natbib}
\usepackage{latexsym}
\usepackage{amsfonts}
\usepackage{eepic}
\usepackage{amsopn}
\usepackage{amsmath,amssymb,rotating,graphicx,rotating,float}

\usepackage{cite}
\usepackage[mathscr]{eucal}
\usepackage{url}
\usepackage{times}

\usepackage{ulem}
\newlength{\dhatheight}

\renewcommand{\d}{{\rm d}} 
\newcommand{\y}{\mathbf{y}} 

\newcommand{\PXY}{\mathcal{P}_{XY}}

\newcommand{\Risk}{{\rm R}}

\newcommand{\RE}{{\rm RE}} 
\newcommand{\AG}{{\rm AG}} 

\newcommand{\D}{\mathcal{D}} 
\newcommand{\metric}{\rho}
\newcommand{\dist}{\metric}

\newcommand{\ERM}{{\rm ERM}}
\newcommand{\SC}{\mathcal{M}} 
\newcommand{\vc}{{\rm vc}}
\renewcommand{\dim}{{\rm dim}} 

\newcommand{\G}{{\cal{G}}}

\newcommand{\polylog}{{\rm polylog}}

\newcommand{\eps}{\varepsilon}

\newcommand{\X}{\mathcal X} 
\newcommand{\Y}{\mathcal Y} 
\newcommand{\alg}{\mathbb{A}} 
\renewcommand{\H}{\mathcal H} 
\renewcommand{\L}{\mathcal L} 
\newcommand{\U}{\mathcal U} 
\newcommand{\Z}{\mathcal{Z}}
\newcommand{\er}{{\rm er}} 

\DeclareSymbolFont{bbold}{U}{bbold}{m}{n}
\DeclareSymbolFontAlphabet{\mathbbold}{bbold}
\newcommand{\ind}{\mathbbold{1}}

\newcommand{\A}{\mathcal{A}}

\renewcommand{\S}{\mathcal S} 
\renewcommand{\P}{\mathbb P} 
\newcommand{\nats}{\mathbb{N}} 
\newcommand{\reals}{\mathbb{R}} 

\newcommand{\E}{\mathbb{E}}

\newcommand{\argmin}{\mathop{\rm argmin}}

\newcommand{\supp}{{\rm supp}} 
\newcommand{\Mre}{\SC_{{\rm RE}}}

\newcommand{\ignore}[1]{}
\newcommand{\todo}[1]{}
\newcommand{\oldstuff}[1]{}

\newsavebox{\savepar}

\makeatletter
\newcommand{\vast}{\bBigg@{3}}
\newcommand{\Vast}{\bBigg@{4}}
\makeatother

\newcommand{\norm}[1]{\lVert#1\rVert}
\newcommand{\RERM}{{\rm RERM}}
\renewcommand{\epsilon}{\eps}

\newcommand{\natinote}[1]{\textrm{\textcolor{red}{[[ {#1} -Nati ]]}}}
\newcommand{\omarnote}[1]{\textrm{\textcolor{red}{[[ {#1} -Omar ]]}}}
\newcommand{\removed}[1]{}

\newcommand{\abs}[1]{\left\lvert{#1}\right\rvert}

\title[Adversarially Robust Learnability]{VC Classes are Adversarially Robustly Learnable, \\but Only Improperly}

\coltauthor{\Name{Omar Montasser} \Email{omar@ttic.edu}\\
  \Name{Steve Hanneke} \Email{steve.hanneke@gmail.com}\\
  \Name{Nathan Srebro} \Email{nati@ttic.edu}\\
  \addr Toyota Technological Institute at Chicago, Chicago IL, USA}

\begin{document}

\maketitle

\begin{abstract}%
We study the question of learning an adversarially robust predictor. We show that any hypothesis class $\H$ with finite VC dimension is robustly PAC learnable with an \emph{improper} learning rule. The requirement of being improper is necessary as we exhibit examples of hypothesis classes $\H$ with finite VC dimension that are \emph{not} robustly PAC learnable with any \emph{proper} learning rule.
\end{abstract}


\section{Introduction}
\label{sec:introduction}

Learning predictors that are robust to adversarial perturbations is an important challenge in contemporary machine learning. There has been a lot of interest lately in how predictors learned by deep learning are {\em not} robust to adversarial examples \citep{szegedy2013intriguing,biggio2013evasion,goodfellow2014explaining}, and there is an ongoing effort to devise methods for learning predictors that {\em are} adversarially robust. \removed{
For example, neural networks have been shown to be vulnerable to adversarial examples in many domains; including image classification \citep{goodfellow2014explaining}, question answering \citep{jia2017adversarial}, and speech recognition \citep{carlini2018audio}. 
As machine learning systems become increasingly integrated into society, it becomes critical to ensure that these systems {\em are} are adversarially robust. 
 }
In this paper, we consider the problem of learning, based on a (non-adversarial) i.i.d.~sample, a predictor that is robust to adversarial examples at test time. We emphasize that this is distinct from the learning process itself being robust to an adversarial training set.

Given an instance space $\X$ and label space $\Y=\{+1,-1\}$,\removed{ and hypothesis class $\H \subseteq \Y^\X$.} we formalize an adversary we would like to protect against as $\U:\X \mapsto 2^\X$, where $\U(x)\subseteq \X$ represents the set of perturbations (adversarial examples) that can be chosen by the adversary at test time. For example, $\U$ could be perturbations of distance at most $\gamma$ w.r.t. some metric $\dist$, such as the $\ell_\infty$ metric considered in many applications: $\U(x)=\{z\in \X: \norm{x-z}_\infty \leq \gamma\}$.
Our only (implicit) restriction on the specification of $\U$ is that $\U(x)$ should be nonempty for every $x$.  
For a distribution $\D$ over $\X\times \Y$, we observe $m$ i.i.d.\ samples $S\sim \D^m$, and our goal is to learn a predictor $\hat{h}: \X \mapsto \Y$ 
having small robust risk\removed{(compared with the best robust predictor in $\H$)}, 
\begin{center}
{\vskip -1mm}$\Risk_{\U}(\hat{h};\D) := \E_{(x,y) \sim \D}\!\left[ \sup\limits_{z\in\U(x)} \ind[\hat{h}(z)\neq y] \right] \removed{\leq \inf_{h \in \H} \Risk_{\U}(h;\D) + \epsilon}$.
\end{center}

{\vskip -1mm}The common approach to adversarially robust learning is to pick a hypothesis class $\H\subseteq \Y^\X$ (e.g.\ neural networks) and learn through
robust {\em empirical} risk minimization:
\begin{center}
{\vskip -2mm}$\hat{h} \in \RERM_\H(S) := \argmin\limits_{h\in \H} \hat{\Risk}_{\U}(h;S)$
\end{center}
%
{\vskip -2mm}where $\hat{\Risk}_{\U}(h;S) = \frac{1}{m} \sum_{(x,y)\in S} \sup_{z\in\U(x)} \ind[h(z)\neq y]$. Most work on the problem has focused on computational approaches to solve this empirical
optimization problem, or related problems of minimizing a robust
version of some surrogate loss instead of the 0/1 loss \citep{madry2017towards,wong2018provable,raghunathan2018certified,raghunathan2018semidefinite}. \removed{ as well as empirical work using such approaches (\omarnote{what should be cited?}]\natinote{Aren't some of the paper, even by Madri, using this approach?}).} But of course our true objective is not the empirical robust risk $\hat{\Risk}_{\U}(h;S)$, but rather the population robust risk $\Risk_{\U}(h;\D)$.

How can we ensure that $\Risk_{\U}(h;\D)$ is small?  All prior approaches that we are aware of for ensuring adversarially robust generalization are based on uniform convergence, i.e.\ showing that w.h.p.\ for all predictors $h \in \H$, the estimation error $\lvert\Risk_{\U}(h;\D) - \hat{\Risk}_{\U}(h;S)\rvert$ is small, perhaps for some surrogate loss \citep{bubeck2018adversarial,cullina2018pac,khim2018adversarial,yin2018rademacher}.  Such approaches justify $\RERM$, and in particular yield M-estimation type \emph{proper} learning rules: we are learning a hypothesis class by choosing a predictor in the class that minimizes some empirical functional.  For standard supervised learning we know that proper learning, and specifically $\ERM$, is sufficient for learning, and so it is sensible to limit attention to such methods.

But it has also been observed in practice that the adversarial error
does not generalize as well as the standard error, i.e.\ there can be a large gap between $\Risk_{\U}(h;\D)$ and $\hat{\Risk}_{\U}(h;S)$ even when their non-robust versions are similar \citep{schmidt2018adversarially}.  This suggests that perhaps the robust risk does not concentrate as well as the standard risk, and so RERM in adversarially robust learning might not work as well as ERM in standard supervised learning. Does this mean that such problems are not adversarially robustly learnable?  Or is it 
perhaps 
that proper 
learners might not be sufficient? 

In this paper we aim to characterize which hypothesis classes are
adversarially robustly learnable, and using what learning rules. That is, for a
given hypothesis class $\H \subseteq \Y^\X$ and adversary $\U$, we ask whether it is possible,
based on an i.i.d.\ sample to learn a predictor $h$ that has population robust risk almost as good as any predictor in $\H$ (see Definition \ref{def:ag_sample_complexity} in Section \ref{sec:notation}). We discover a stark contrast between \emph{proper} learning rules which output predictors in $\H$, and \emph{improper} learning rules which are not constrained to predictors in $\H$. Our main results are:
\begin{itemize}
    \item We show that there exists an adversary $\U$ and a hypothesis class $\H$ with finite VC dimension that \emph{cannot} be robustly PAC learned with any \emph{proper} learning rule (including $\RERM$).
    \item We show that for any adversary $\U$ and any hypothesis class $\H$ with finite VC dimension, there exists an \emph{improper} learning rule that can robustly PAC learn $\H$ (although with sample complexity that is sometimes exponential in the VC dimension).
\end{itemize}

Our results suggest that we should start considering \emph{improper} learning rules to ensure adversarially robust generalization. They also demonstrate that previous approaches to adversarially robust generalization are not always sufficient, as all prior work we are aware of is based on uniform convergence of the robust risk, either directly for the loss of interest \citep{bubeck2018adversarial,cullina2018pac} or some carefully constructed surrogate loss \citep{khim2018adversarial,yin2018rademacher}, which would still justify the use of M-estimation type proper learning.  The approach of \citet{attias2018improved} for the case where $|\U(x)|\leq k$ (i.e.\ finite number of perturbations) is most similar to ours, as it uses an improper learning rule, but their analysis is still based on uniform convergence and so would apply also to $\RERM$ (the improperness is introduced only for computational, not statistical, reasons).  Also, in this specific case, our approach would give an improved sample complexity that scales only roughly logarithmically with $k$, as opposed to the roughly linear scaling in \citet{attias2018improved}---see discussion at the end of Section \ref{sec:vc-dim} for details.

A related negative result was presented by \citet{schmidt2018adversarially}, where they showed that there exists a family of distributions (namely, mixtures of two $d$-dimensional spherical Gaussians) where the sample complexity for standard learning is $O(1)$, but the sample complexity for adversarially robust learning is at least $\Omega(\frac{\sqrt{d}}{\log d})$. This an interesting instance where there is a large separation in sample complexity between standard learning and robust learning.  But distribution-specific learning is known to be less easily characterizable, with the uniform convergence not being necessary for learning, and ERM not always being optimal, even for standard (non-robust) supervised learning.  In this paper we focus on ``worst case'' distribution-free robust learning, as in standard PAC learnability.

A different notion of robust learning was studied by \citet{xu2012robustness}.
They use empirical robustness as a design technique for learning rules, but
their goal, and the guarantees they establish are on the standard
non-robust population risk, and so do not inform us about robust
learnability.

\section{Problem Setup}
\label{sec:notation}
We are interested in studying the sample complexity of adversarially robust PAC learning in the realizable and agnostic settings. Given a hypothesis class $\H \subseteq \Y^\X$, our goal is to design a learning rule $\A:(\X\times \Y)^* \mapsto \Y^\X$ such that for any distribution $\D$ over $\X\times \Y$, the rule $\A$ will find a predictor that competes with the best predictor $h^*\in \H$ in terms of the robust risk using a number of samples that is independent of the distribution $\D$.
The following definitions formalize the notion of robust PAC learning in the realizable and agnostic settings:\footnote{We implicitly suppose 
that the hypotheses $h$ in $\H$ and their losses $\sup_{z \in \U(x)} \ind[h(z)\neq y]$ are measurable, 
and that standard mild restrictions on $\H$ are imposed to guarantee measurability of 
empirical processes, so that the standard tools of VC theory apply.  See \citet*{blumer:89,van-der-Vaart:96} for discussion of such measurability 
issues, which we will not mention again in the remainder of this article.}

\begin{definition}[Agnostic Robust PAC Learnability]
\label{def:ag_sample_complexity}
For any $\epsilon,\delta \in (0,1)$, the \emph{sample complexity of agnostic robust $(\epsilon, \delta)-$PAC learning of $\H$ with respect to adversary $\U$}, denoted $\SC_{\AG}(\eps,\delta;\H,\U)$, is defined as the smallest $m \in \nats \cup \{0\}$ for which there exists a learning rule $\A: (\X\times \Y)^* \mapsto \Y^\X$ such that, for every data distribution $\D$ over $\X \times \Y$, with probability at least $1-\delta$ over $S \sim \D^m$,
\begin{center}
{\vskip -1mm}$\Risk_{\U}(\A(S);\D)\leq \inf\limits_{h \in \H} \Risk_{\U}(h;\D) + \epsilon$.
\end{center}
{\vskip -2mm}If no such $m$ exists, define $\SC_{\AG}(\eps,\delta;\H,\U) = \infty$. We say that $\H$ is robustly PAC learnable in the agnostic setting with respect to adversary $\U$ if $\forall \epsilon,\delta \in (0,1)$, 
$\SC_{\AG}(\eps,\delta;\H,\U)$ is finite.
\end{definition}

\begin{definition}[Realizable Robust PAC Learnability]
\label{def:re_sample_complexity}
For any $\epsilon,\delta \in (0,1)$, the \emph{sample complexity of realizable robust $(\epsilon, \delta)$-PAC learning of $\H$ with respect to adversary $\U$}, denoted $\SC_{\RE}(\eps,\delta;\H,\U)$, is defined as the smallest $m \in \nats \cup \{0\}$ for which there exists a learning rule $\A: (\X\times \Y)^* \mapsto \Y^\X$ such that, for every data distribution $\D$ over $\X \times \Y$ where there exists a predictor $h^*\in \H$ with zero robust risk, $\Risk_{\U}(h^*;\D) = 0$, with probability at least $1-\delta$ over $S \sim \D^m$,
\begin{center}
{\vskip -1mm}$\Risk_{\U}(\A(S);\D) \leq \epsilon$.
\end{center}
{\vskip -1mm}If no such $m$ exists, define $\SC_{\RE}(\eps,\delta;\H,\U)) = \infty$. We say that $\H$ is robustly PAC learnable in the realizable setting with respect to adversary $\U$ if $\forall \epsilon,\delta \in (0,1)$, 
$\SC_{\RE}(\eps,\delta;\H,\U)$ is finite.
\end{definition}

\begin{definition}[Proper Learnability] We say that $\H$ is \emph{properly} robustly PAC learnable (in the agnostic or realizable setting) if it can be learned as in Definitions \ref{def:ag_sample_complexity} or \ref{def:re_sample_complexity} using a learning rule $\A: (\X\times \Y)^* \mapsto \H$ that always outputs a predictor in $\H$.  We refer to learning using any learning rule $\A:(\X\times \Y)^* \mapsto \Y^\X$, as in the definitions above, as {\em improper} learning.    
\end{definition}

We also denote $\er(h;\D) = \P( h(x) \neq y )$, the (non-robust) error rate under the $0$-$1$ loss, and $\hat{\er}(h;S) = \frac{1}{|S|} \sum_{(x,y) \in S} \ind[ h(x) \neq y ]$ the empirical error rate.  These agree with the robust variant when $\U(x)=\{ x \}$, and so robust learnability agrees with standard supervised learning when $\U(x)=\{ x\}$.  For more powerful adversaries, robust learnability is a special case of Vapink's ``General Learning'' \citep{vapnik:82}, but can not, in general, be phrased in terms of supervised learning of some modified hypothesis class or loss.  We recall the Vapnik-Chervonenkis dimension (VC dimension) is defined as follows,

\begin{definition}[VC dimension]
\label{VCdim}
We say that a sequence $\{x_1,\dots,x_k\}\in\X$ is shattered by $\H$ if $\forall y_1,\dots,y_k\in \Y, \exists h\in \H$ such that $\forall i\in[k], h(x_i)=y_i$. The VC dimension of $\H$ (denoted $\vc(\H)$) is then defined as the largest integer $k$ for which there exists $\{x_1,\dots,x_k\}\in\X$ that is shattered by $\H$. If no such $k$ exists, then $\vc(\H)$ is said to be infinite.
\end{definition}

In the standard PAC learning framework, we know that a hypothesis class $\H$ is PAC learnable if and only if the VC dimension of $\H$ is finite \citep{vapnik:71,vapnik:74,blumer:89,ehrenfeucht:89}. In particular, $\H$ is properly PAC learnable with $\ERM_{\H}$ and therefore proper learning is sufficient for supervised learning. A natural question to ask, based on the definition of robust PAC learning, is what is a necessary and sufficient condition on $\H$ that implies that it is robustly PAC learnable with respect to adversary $\U$. We can easily obtain a sufficient condition based on Vapink's ``General Learning'' \citep{vapnik:82}. Denote by $\L^{\U}_\H$ the robust loss class of $\H$,
\begin{center}
{\vskip -2mm}$\L^{\U}_\H = \left\{(x,y)\mapsto \sup\limits_{z\in\U(x)} \ind[h(z)\neq y] : h\in \H \right\}$.
\end{center}

{\vskip -1mm}If the robust loss class $\L_{\H}^{\U}$ has finite VC dimension ($\vc(\L_{\H}^{\U})<\infty$), then $\H$ is robustly PAC learnable with $\RERM_\H$ and sample complexity that scales linearly with $\vc(\L_{\H}^{\U})$.  One might then wish to relate the VC dimension of the hypothesis class ($\vc(\H)$) to the VC dimension of the robust loss class ($\vc(\L_\H^\U)$).  But as we show in Sections \ref{sec:proper} and \ref{sec:adjacent}, there can be arbitrarily large gaps between them.

As mentioned earlier, for supervised learning finite VC dimension of the loss class (which is equal to the VC dimension of the hypothesis class) is also necessary for learning.  For general learning, unlike supervised learning, the loss class having finite VC dimension, and uniform convergence over this class, is not, in general, necessary, and rules other than $\ERM$ might be needed for learning \citep[e.g.][]{vapnik:82,shalev2009stochastic,daniely:15}. In the following Sections, we show that this is also the case for robust learning. We show that $\vc(\L_{\H}^{\U})$ can be arbitrarily larger, we might not have uniform convergence, $\RERM$ might not ensure learning, while the problem is still learnable with a different (improper, in our case) learning rule.

\section{Sometimes There are no Proper Robust Learners}
\label{sec:proper}

We start by showing that even for hypothesis classes with finite VC dimension, indeed even if $\vc(\H)=1$, robust PAC learning might not be possible using {\em any} proper learning rule.  In particular, even if there is a robust predictor in $\H$, and even with an unbounded number of samples, $\RERM$ (or any other M-estimator or other proper learning rules), will not ensure a low robust risk.

\begin{theorem}
\label{thm:pac_properfail}
There exists a hypothesis class $\H \subseteq \Y^\X$ with $\vc(\H)\leq 1$ and an adversary $\U$ such that $\H$ is not properly robustly PAC learnable with respect to $\U$ in the realizable setting. 
\end{theorem}

This result implies that finite VC dimension of a hypothesis class $\H$ is not sufficient for robust PAC learning if we want to use \emph{proper} learning rules. For the proofs in this section, we will fix an instance space $\X=\reals^d$ equipped with a metric $\dist$, and an adversary $\U: \X \mapsto 2^\X$ such that $\U(x)=\{z\in \X: \dist{(x,z)} \leq \gamma\}$ for all $x\in\X$ for some $\gamma > 0$. First, we prove a lemma that shows that there exists a hypothesis class $\H$ where there is an arbitrarily large gap between the VC dimension of $\H$ and the VC dimension of the robust loss class of $\H$,

\begin{lemma}
\label{lemma:vcblowup}
Let $m\in \nats$. Then, there exists $\H \subseteq \Y^\X$ such that $\vc{(\H)}\leq 1$ but $\vc{(\L^{\U}_{\H})}\geq m$. 
\end{lemma}

{\vskip -1mm}\begin{proof}
Pick $m$ points $x_1, \dots, x_m$ in $\X$ such that for all $i,j\in[m], \U(x_i)\cap\U(x_j)=\emptyset$. In other words, we want the perturbation sets $\U(x_1),\dots,\U(x_m)$ to be mutually disjoint.

We will construct a hypothesis class $\H$ in the following iterative manner. Initialize set $\Z=\{x_1,\dots,x_m\}$. For each bit string $b\in \{0,1\}^m$, initialize $Z_b=\emptyset$. For each $i\in[m]$, if $b_i=1$ then pick a point $z \in \U(x_i)\setminus \Z$ and add it to $Z_b$, i.e. $Z_b=Z_b\cup\{z\}$. Once we finish picking points based on all bits that are set to $1$, we add $Z_b$ to $\Z$ (i.e. $\Z=\Z \cup Z_b$). We define $h_b:\X \rightarrow \Y$ as:
\begin{center}
    {\vskip -2mm}$h_b(x) = \left\{
        \begin{array}{ll}
            +1 & \text{if } x \notin Z_b \\
            -1 & \text{if } x \in Z_b
        \end{array}
    \right. $
\end{center}
{\vskip -2mm}Then, let $\H=\{h_b: b\in \{0,1\}^m\}$. We can think of each mapping $h_b$ as being characterized by a unique signature $Z_b$ that indicates the points that it labels with $-1$. These points are carefully picked such that, first, they are inside the perturbation sets of $x_1,\dots,x_m$; and second, no two mappings label the same point with $-1$, i.e. for any $b,b'\in \{0,1\}^m$, where $b \neq b'$, $Z_b \cap Z_b' = \emptyset$. Also, we make sure that all mappings in $\H$ label the set $\{x_1,\dots,x_m\}$ with $+1$.

Next, we proceed with proving two claims about $\H$. First, that $\vc(\H)\leq1$. Pick any two points $z_1,z_2\in\X$. Consider the following cases. In case $z_1$ or $z_2$ is in $\X \setminus \Z$. Suppose W.L.O.G that $z_2 \in \X \setminus \Z$. Then we know that all mappings label $z_2$ in the same way with label $+1$, because for all $b\in \{0,1\}^m, z_2 \notin Z_b$. Therefore, we cannot shatter $z_1,z_2$ with $\H$. In case $z_1$ and $z_2$ are both in $\Z$. Since by our construction, $\Z=\cup_{b\in\{0,1\}^m} Z_b$ and $Z_b \cap Z_b' = \emptyset$ for any $b\neq b'$, we have two sub-cases. Either $z_1,z_2\in Z_b$ for some $b\in \{0,1\}^m$, which means that the only labelings we can obtain are $(-1,-1)$ with $h_b$, and $(+1,+1)$ with $h_b'$ for any $b'\neq b$. Second case is that $z_1 \in Z_b$ and $z_2 \in Z_{b'}$ for $b\neq b', b,b'\in \{0,1\}^m$. By our construction, we know that we cannot label both points $z_1$ and $z_2$ with $(-1,-1)$, because they don't belong to the same set. Therefore, in both subcases, we cannot shatter $z_1,z_2$ with $\H$. This concludes that $\vc{(\H)}\leq 1$.

Second, we will show that $\vc(\L^{\U}_\H)\geq m$. Consider the set $S=\{(x_1,+),\dots,(x_m,+)\}$. We will show that $\L^{\U}_\H$ shatters $S$. Pick any labeling $y\in \{0,1\}^m$. Note that by construction of $\H$, $\exists h_b \in \H$ such that $b=y$. Then, for each $i\in[m]$, $\sup_{z \in \U(x_i)} \ind[ h_b(z) \neq +1 ] = b_i = y_i$. This shows that $\L^{\U}_\H$ shatters $S$, and therefore $\vc{(\L^{\U}_\H)}\geq m$.
\end{proof}

The following lemma (proof provided in Appendix \ref{appendix-proper}) establishes that for any sample size $m\in\nats$, there exists a hypothesis class $\H$ with $\vc{(\H)}\leq 1$ such that any \emph{proper} learning rule will fail in learning a robust classifier if it observes at most $m$ samples but not more. 

\begin{lemma}
\label{lemma:properfail_realizable}
Let $m\in \nats$. Then, there exists $\H \subseteq \Y^\X$ with $\vc{(\H)}\leq 1$ such that for any proper learning rule $\A: (\X \times \Y)^* \mapsto \H$,
\begin{itemize}
    \item $\exists$ a distribution $\D$ over $\X \times \Y$ and a predictor $h^*\in\H$ where $\Risk_{\U}(h^*;\D) = 0$.
   \item With probability at least $1/7$ over $S\sim \D^m$, $\Risk_{\U}(\A(S);\D) > 1/8$.
\end{itemize}
\end{lemma}

{\vskip -2mm}We now proceed with the proof of Theorem~\ref{thm:pac_properfail}.

\begin{proof} [of Theorem~\ref{thm:pac_properfail}]
Let $(X_m)_{m\in \nats}$ be an infinite sequence of sets such that each set $X_m$ contains $3m$ distinct points from $\X$, where for any $x_i, x_j \in \cup_{m=1}^{\infty} X_m$ such that $x_i\neq x_j$ we have $\U(x_i)\cap\U(x_j)=\emptyset$. Foreach $m\in\nats$, construct $\H_m$ on $X_m$ as in Lemma \ref{lemma:properfail_realizable}. We want to ensure that predictors in $\H_m$ are non-robust on the points in $X_{m'}$ for all $m'\neq m$, by doing the following adjustment for each $h_b \in \H_m$ (recall from Lemma \ref{lemma:vcblowup} that each predictor has its own unique signature $Z_b$),
\begin{center}
    {\vskip -1mm}$h_b(x) = \left\{
        \begin{array}{ll}
            -1 & \text{if } x \in Z_b \text{ or } x \in X_{m'} \text{ for } m'\neq m \\
            +1 & \text{otherwise }
        \end{array}
    \right.$
\end{center}
{\vskip -1mm}Let $\H=\cup_{m=1}^{\infty} \H_m$. We will show that $\vc(\H)\leq 1$. Pick any two points $z_1,z_2\in\X$. There are six cases to consider. In case both $z_1$ and $z_2$ are in $X_m$ for some $m\in\nats$, then we only obtain the labelings $(+1,+1)$ (by predictors from $\H_m$) and $(-1,-1)$ (by predictors from $\H_{m'}$ with $m'\neq m$). In case both $z_1$ and $z_2$ are in $\U(X_m)\setminus X_m$, then they are not shattered by Lemma \ref{lemma:vcblowup}. In case $z_1\in X_i$ and $z_2 \in X_j$ for $i\neq j$, then we can only obtain the labelings $(+1,-1)$ (by predictors in $\H_i$), $(-1,+1)$ (by predictors in $\H_j$), and $(-1,-1)$ (by predictors in $\H_k$ for $k\neq i,j$). In case $z_1\in X_i$ and $z_2 \in \U(X_j)\setminus X_j$ for $j\neq i$, then we can't obtain the labeling $(+1,-1)$. In case $z_1\in \U(X_i)\setminus X_i$ and $z_2 \in \U(X_j) \setminus X_j$ for $i\neq j$, then we can't obtain the labeling $(-1,-1)$. Finally, if either $z_1$ or $z_2$ is in $\X$ but not in $\cup_{m=1}^{\infty}X_m$ and not in $\cup_{m=1}^{\infty}\U(X_m)$, then all predictors label $z_1$ or $z_2$ with $+1$, and so we can't shatter them. This shows that $\vc(\H)\leq 1$.

By Lemma $\ref{lemma:properfail_realizable}$, it follows that for any proper learning rule $\A:(\X \times \Y)^*\mapsto \H$ and for any $m\in \nats$, we can construct a distribution $\D$ over $X_m \times \Y$ where there exists a predictor $h^*\in\H_m$ with $\Risk_{\U}(h^*;\D) = 0$, but with probability at least $1/7$ over $S\sim \D^m$, $\Risk_{\U}(\A(S);\D) > 1/8$. This works because classifiers from classes $\H_{m'}$ where $m'\neq m$ make mistakes on points in $X_m$ and so they are non-robust. Thus, rule $\A$ will do worse if it picks predictors from these classes. This shows that the sample complexity to properly robustly PAC learn $\H$ is infinite. This concludes that $\H$ is not properly robustly PAC learnable.
\end{proof}

\section{Finite VC Dimension is Sufficient for (Improper) Robust Learnability}
\label{sec:vc-dim}

{\vskip -2mm}In the previous section we saw that finite VC dimension is {\em not} sufficient for {\em proper} robust learnability.  We now show that it {\em is} sufficient for {\em improper} robust learnability, thus (1) establishing that if $\H$ is learnable, it is also robustly learnable, albeit possibly with a higher sample complexity; and (2) unlike the standard supervised learning setting, to achieve learnability we might need to escape properness, as improper learning is necessary for some hypothesis classes.

We begin, in Section \ref{subsec:realizable} with the realizable case, i.e. where there exists $h^*\in \H$ with zero robust risk. Then in Section \ref{subsec:agnostic} we turn to the agnostic setting, and observe that a version of a recent reduction by \citet*{david:16} from agnostic to realizable learning applies also for robust learning.  We thus establish agnostic robust learnability of finite VC classes by using this reduction and relying on the realizable learning result of Section \ref{subsec:realizable}.

\subsection{Realizable Robust Learnability}
\label{subsec:realizable}


{\vskip -2mm}We will in fact establish a bound in terms 
of the \emph{dual VC dimension}.  Formally, for each $x \in \X$,
define a function $g_{x} : \H \to \Y$ such that $g_{x}(h) = h(x)$ for each $h \in \H$.
Then the dual VC dimension of $\H$, denoted $\vc^{*}(\H)$, is defined as the VC dimension 
of the set $\G = \{ g_{x} : x \in \X \}$. 
This quantity is known to satisfy $\vc^{*}(\H) < 2^{\vc(\H)+1}$ \citep*{assouad:83}, 
though for many spaces it satisfies $\vc^{*}(\H) = O({\rm poly}(\vc(\H)))$ or even, as is the case for linear separators, $\vc^{*}(\H) = O(\vc(\H))$.

{\vskip -6mm}\begin{theorem}
\label{thm:vc-dim}
For any $\H$ and $\U$, $\forall \eps,\delta \in (0,1/2)$, 
\begin{equation*}
\SC_{\RE}(\eps,\delta;\H,\U) = O\!\left( \vc(\H) \vc^{*}(\H) \frac{1}{\eps} \log\!\left(\frac{\vc(\H) \vc^{*}(\H)}{\eps}\right) + \frac{1}{\eps}\log\!\left(\frac{1}{\delta}\right) \right),
\end{equation*}
\end{theorem}

Since \citet*{assouad:83} has shown $\vc^{*}(\H) < 2^{\vc(\H)+1}$, 
this implies the following corollary.

{\vskip -6mm}\begin{corollary}
\label{cor:vc-dim}
For any $\H$ and $\U$, $\forall \eps,\delta \in (0,1/2)$, 
\begin{equation*}
\SC_{\RE}(\eps,\delta;\H,\U) = 2^{O(\vc(\H))} \frac{1}{\eps} \log\!\left(\frac{1}{\eps}\right) + O\!\left(\frac{1}{\eps}\log\!\left(\frac{1}{\delta}\right) \right).
\end{equation*}
\end{corollary}


Our approach to this proof is via \emph{sample compression} arguments.
Specifically, we make use of a lemma (Lemma~\ref{lem:robust-compression} in Appendix~\ref{subsec:agnostic}), 
which extends to the robust loss 
the classic compression-based generalization guarantees from the $0$-$1$ loss. 
We now proceed with the proof of Theorem~\ref{thm:vc-dim}.

\begin{proof}[of Theorem~\ref{thm:vc-dim}]
The learning algorithm achieving this bound is a modification of a sample compression scheme 
recently proposed by \citet*{moran:16}, or more precisely, a variant of that method explored by \citet*{hanneke:19a}.  Our modification forces the compression scheme to also have zero empirical \emph{robust} loss.
Fix $\eps,\delta \in (0,1)$ and a sample size $m > 2 \vc(\H)$, and denote by $P$ any distribution with $\inf_{h \in \H} R_{\U}(h;P) = 0$.

By classic PAC learning guarantees \citep{vapnik:74,blumer:89}, 
there is a positive integer $n = O(\vc(\H))$ with the property that, 
for any distribution $D$ over $\X \times \Y$ with $\inf_{h \in \H} \er(h;D)=0$, 
for $n$ iid $D$-distributed samples $S^{\prime} = \{(x_{1}^{\prime},y_{1}^{\prime}),\ldots,(x_{n}^{\prime},y_{n}^{\prime})\}$,
with nonzero probability, 
every $h \in \H$ satisfying $\hat{\er}(h;S^{\prime}) = 0$
also has $\er(h;D) < 1/3$.

Fix a deterministic function $\RERM_{\H}$ mapping any labeled data set to a classifier in $\H$ robustly consistent with the labels in the data set, 
if a robustly consistent classifier exists (i.e., having zero $\hat{\Risk}_{\U}$ on the given data set).
Suppose we are given training examples $S = \{(x_{1},y_{1}),\ldots,(x_{m},y_{m})\}$ as input to the learner.
Under the assumption that this is an iid sample from a robustly realizable distribution, we suppose $\hat{R}_{\U}(\RERM_{\H}(S);S)=0$, 
which should hold with probability one.
Denote by $I(x) = \min\{ i \in \{1,\ldots,m\} : x \in \U(x_{i}) \}$ for every $x \in \bigcup_{i \leq m} \U(x_{i})$. 
Before we can apply the compression approach, we first need to \emph{inflate} the data set to a (potentially infinite) larger set, 
and then \emph{discretize} it to again reduce it back to a finite sample size.
Denote by $\hat{\H} = \{ \RERM_{\H}(L) : L \subseteq S, |L| = n \}$.
Note that $|\hat{\H}| \leq |\{ L : L \subseteq S, |L| = n \}| = \binom{m}{n} \leq \left(\frac{e m}{n}\right)^{n}$. 
Define an \emph{inflated} data set $S_{\U} = \bigcup_{i \leq m} \{ (x,y_{I(x)}) : x \in \U(x_{i}) \}$.
As it is difficult to handle this potentially-infinite set in an algorithm, we consider a discretized version of it.
Specifically, consider a \emph{dual space} $\G$: a set of functions $g_{(x,y)} : \H \to \{0,1\}$ defined as $g_{(x,y)}(h) = \ind[ h(x) \neq y ]$, 
for each $h \in \H$ and each $(x,y) \in S_{\U}$.  
The VC dimension of $\G$ is at most the \emph{dual VC dimension} of $\H$: $\vc^{*}(\H)$, 
which is known to satisfy $\vc^{*}(\H) < 2^{\vc(\H)+1}$ \citep*{assouad:83}.
Now denote by $\hat{S}_{\U}$ a subset of $S_{\U}$ which includes exactly one $(x,y) \in S_{\U}$ 
for each distinct classification $\{g_{(x,y)}(h)\}_{h \in \hat{\H}}$ of $\hat{\H}$ realized by functions $g_{(x,y)} \in \G$.
In particular, by Sauer's lemma \citep*{vapnik:71,sauer:72}, 
$|\hat{S}_{\U}| \leq \left(\frac{e |\hat{\H}|}{\vc^{*}(\H)}\right)^{\vc^{*}(\H)}$, 
which for $m > 2 \vc(\H)$ is at most $\left( e^{2} m/\vc(\H) \right)^{\vc(\H) \vc^{*}(\H)}$.
In particular, note that for any $T \in \nats$ and $h_{1},\ldots,h_{T} \in \hat{\H}$, 
if $\frac{1}{T} \sum_{t=1}^{T} \ind[ h_{t}(x) = y ] > \frac{1}{2}$ for every $(x,y) \in \hat{S}_{\U}$, 
then $\frac{1}{T} \sum_{t=1}^{T} \ind[ h_{t}(x) = y ] > \frac{1}{2}$ for every $(x,y) \in S_{\U}$ as well, 
which would further imply $\hat{R}_{\U}( {\rm Majority}(h_{1},\ldots,h_{T}); S) = 0$.
We will next go about finding such a set of $h_{t}$ functions. 

By our choice of $n$, we know that for any distribution $D$ over $\hat{S}_{\U}$, 
$n$ iid samples $S^{\prime}$ sampled from $D$ would have the property that, with nonzero probability, 
all $h \in \H$ with $\hat{\er}(h;S^{\prime}) = 0$ also have $\er(h;D) < 1/3$.
In particular, this implies at least that there \emph{exists} a subset $S^{\prime} \subseteq \hat{S}_{\U}$
with $|S^{\prime}| \leq n$
such that every $h \in \H$ with $\hat{\er}(h;S^{\prime})=0$ has $\er(h;D) < 1/3$.
For such a set $S^{\prime}$, note that $\{ (x_{I(x)},y) : (x,y) \in S^{\prime} \} \subseteq S$, 
and therefore there exists a set $L$ with $|L|=n$ and $\{ (x_{I(x)},y) : (x,y) \in S^{\prime} \} \subseteq L \subseteq S$. 
Furthermore, since $x \in \U(x_{I(x)})$ for every $(x,y) \in S^{\prime}$, we know $\hat{\er}(\RERM_{\H}(L);S^{\prime}) = 0$, 
and hence $\er(\RERM_{\H}(L);D) < 1/3$.
Altogether, we have that, for any distribution $D$ over $\hat{S}_{\U}$, 
$\exists h_{D} \in \hat{\H}$ with $\er(h_{D};D) < 1/3$.

We will use the above $h_{D}$ as a \emph{weak hypothesis} in a boosting algorithm.
Specifically, we run the $\alpha$-Boost algorithm  
\citep*[][Section 6.4.2]{schapire:12} with $\hat{S}_{\U}$ as its data set, 
using the above mapping to produce the weak hypotheses for the distributions $D_{t}$ produced on each round of the algorithm.  
As proven in \citep*{schapire:12}, for an appropriate a-priori choice of $\alpha$ in the $\alpha$-Boost algorithm, 
running this algorithm for $T = O(\log(|\hat{S}_{\U}|))$ rounds suffices to produce 
a sequence of hypotheses $\hat{h}_{1},\ldots,\hat{h}_{T} \in \hat{\H}$ s.t.  
\begin{center}
{\vskip -1mm}$\forall (x,y) \in \hat{S}_{\U}, \frac{1}{T} \sum_{i=1}^{T} \ind[ h_{i}(x) = y ] \geq \frac{5}{9}$.
\end{center}

{\vskip -2mm}From this observation, we already have a sample complexity bound, only slightly worse than the claimed result.
Specifically, the above implies that $\hat{h} = {\rm Majority}(\hat{h}_{1},\ldots,\hat{h}_{T})$ 
satisfies $\hat{R}_{\U}(\hat{h};S) = 0$.
Note that each of these classifiers $\hat{h}_{t}$ is equal $\RERM_{\H}(L_{t})$ for some $L_{t} \subseteq S$ with $|L_{t}|=n$.
Thus, the classifier $\hat{h}$ is representable as the value of an (order-dependent) reconstruction function $\phi$ with 
a compression set size
\begin{equation}
\label{eqn:ub-intermediate-mk}
nT = O(\vc(\H) \log(|\hat{S}_{\U}|)) = 
O( \vc(\H)^{2} \vc^{*}(\H) \log(m/\vc(\H)) ).
\end{equation}
{\vskip -2mm}Thus, invoking Lemma~\ref{lem:robust-compression}, if $m > c \vc(\H)^{2} \vc^{*}(\H) \log(\vc(\H) \vc^{*}(\H))$ (for a sufficiently large numerical constant $c$), 
we have that with probability at least $1-\delta$, 
\begin{center}
{\vskip -2mm}$R_{\U}(\hat{h};P) \leq O\!\left( \vc(\H)^{2} \vc^{*}(\H) \frac{1}{m} \log(m/\vc(\H)) \log(m) + \frac{1}{m} \log(1/\delta) \right)$,
\end{center}
{\vskip -2mm}and setting this less than $\eps$ and solving for a sufficient size of $m$ to achieve this yields a 
sample complexity bound, which is slightly larger than that claimed in Theorem~\ref{thm:vc-dim}.
We next proceed to further refine this bound via a sparsification step.
However, as an aside, we note that the above intermediate step will be useful in a discussion below, 
where the size of this compression scheme in the second expression in \eqref{eqn:ub-intermediate-mk} 
offers an improvement over a 
result of \citet*{attias2018improved}.

Via a technique of \citep*{moran:16} we can further reduce the above bound.
Specifically, since all of $\hat{h}_{1},\ldots,\hat{h}_{T}$ are in $\H$, 
classic uniform convergence results of \citet*{vapnik:71} 
imply that taking $N = O(\vc^{*}(\H))$ independent random indices $i_{1},\ldots,i_{N} \sim {\rm Uniform}(\{1,\ldots,T\})$, 
we have 
$\sup\limits_{(x,y) \in \X \times \Y}  \left| \frac{1}{N} \sum\limits_{j=1}^{N} \ind[ h_{i_{j}}(x) = y ] - \frac{1}{T} \sum\limits_{i=1}^{T} \ind[ h_{i}(x) = y ] \right| < \frac{1}{18}$.
In particular, together with the above guarantee from $\alpha$-Boost, 
this implies that there exist indices $i_{1},\ldots,i_{N} \in \{1,\ldots,T\}$ (which may be chosen deterministically) satisfying 
\begin{center}
{\vskip -2mm}$\forall (x,y) \in \hat{S}_{\U}, \frac{1}{T} \sum_{j=1}^{N} \ind[ h_{i_{j}}(x) = y ] \geq -\frac{1}{18} + \frac{1}{T} \sum_{i=1}^{T} \ind[ h_{i}(x) = y ] > -\frac{1}{18}+\frac{5}{9} = \frac{1}{2}$,
\end{center}
{\vskip -2mm}so that the majority vote predictor $\hat{h}^{\prime}(x) = {\rm Majority}(\hat{h}_{i_{1}},\ldots,\hat{h}_{i_{N}})$
satisfies $\hat{\er}(\hat{h}^{\prime};\hat{S}_{\U}) = 0$, and hence $\hat{R}_{\U}(\hat{h}^{\prime};S) = 0$.
Since again, each $\hat{h}_{i_{j}}$ is the result of $\RERM_{\H}(L_{i_{j}})$ for some $L_{i_{j}} \subseteq S$ of size $n$,  
we have that $\hat{h}^{\prime}$ can be represented as the value of an (order-dependent) reconstruction function $\phi$ 
with a compression set size $nN = O(\vc(\H)) \vc^{*}(\H))$.  Thus, Lemma~\ref{lem:robust-compression} 
implies that, for $m \geq c \vc(\H) \vc^{*}(\H)$ (for an appropriately large numerical constant $c$), with probability at least $1-\delta$, 
$R_{\U}(\hat{h}^{\prime};P) \leq O\!\left( \vc(\H) \vc^{*}(\H) \frac{1}{m} \log(m) + \frac{1}{m} \log(1/\delta) \right)$.
Setting this less than $\eps$ and solving for a sufficient size of $m$ to achieve this yields the stated bound.
\end{proof}

\subsection{Agnostic Robust Learnability}
\label{subsec:agnostic}

{\vskip -2mm}For the agnostic case, we can establish an upper bound via reduction 
to the realizable case, following an argument from  
\citet*{david:16}.
Specifically, we have the following result.
%

{\vskip -5mm}\begin{theorem}
\label{thm:vc-dim-agnostic-ub}
For any $\H$ and $\U$, $\forall \eps,\delta \in (0,1/2)$, 
{\vskip -4mm}\begin{equation*}
\SC_{\AG}(\eps,\delta;\H,\U) = O\!\left( \vc(\H) \vc^{*}\!(\H) \log(\vc(\H) \vc^{*}\!(\H)) \tfrac{1}{\eps^{2}} \log^{2}\!\!\left(\tfrac{\vc(\H)\vc^{*}\!(\H)}{\eps}\right) + \tfrac{1}{\eps^{2}}\!\log\!\left(\tfrac{1}{\delta}\right) \right).
\end{equation*}
\end{theorem}

{\vskip -1mm}\noindent As above, since \citet*{assouad:83} has shown $\vc^{*}(\H) < 2^{\vc(\H)+1}$, 
this implies the following corollary.

{\vskip -5mm}\begin{corollary}
\label{cor:vc-dim-agnostic-ub}
For any $\H$ and $\U$, $\forall \eps,\delta \in (0,1/2)$, 
{\vskip -3mm}\begin{equation*}
\SC_{\AG}(\eps,\delta;\H,\U) = 2^{O(\vc(\H))} \frac{1}{\eps^{2}} \log^{2}\!\left(\frac{1}{\eps}\right) + O\!\left(\frac{1}{\eps^{2}}\log\!\left(\frac{1}{\delta}\right) \right).
\end{equation*}
\end{corollary}

%
{\vskip -1mm}We establish the theorem via a reduction to the realizable case, 
following an approach used by \citet*{david:16}, except here applied 
to the robust loss.  The reduction is summarized in the following Theorem, whose proof can be found in Appendix~\ref{appendix-agnostic}:

{\vskip -5mm}\begin{theorem}
\label{thm:agnostic-reduction}
Denote $\Mre = \SC_{\RE}(1/3,1/3;\H,\U)$.
Then 
{\vskip -2mm}\begin{equation*}
\SC_{\AG}(\eps,\delta;\H,\U) 
= O\!\left( \frac{\Mre}{\eps^{2}} \log^{2}\!\left(\frac{\Mre}{\eps}\right) + \frac{1}{\eps^{2}}\log\!\left(\frac{1}{\delta}\right) \right).
\end{equation*}
\end{theorem}

{\vskip -1mm}\noindent From this, Theorem~\ref{thm:vc-dim-agnostic-ub} follows immediately 
by combining Theorem~\ref{thm:agnostic-reduction} with Theorem \ref{thm:vc-dim}.

{\vskip -1mm}\paragraph{Bounded cardinality confusion sets:} As noted in the proof of Theorem~\ref{thm:vc-dim}, 
the compression size \eqref{eqn:ub-intermediate-mk} 
further implies an improvement over a theorem of \citet*{attias2018improved}.
Specifically, \citeauthor{attias2018improved}~considered the case $\max_{x \in \X} |\U(x)| \leq k$ for some fixed $k \in \nats$,
and presented a learning rule establishing the sample complexity gurantee:
{\vskip -2mm}\begin{equation}\label{eq:attias}
\SC_{\AG}(\eps,\delta;\H,\U) = O\!\left( \tfrac{\vc(\H) k \log(k)}{\eps^{2}}+ \tfrac{1}{\eps^2}\log\!\left(\tfrac{1}{\delta}\right) \right).
\end{equation}
Their analysis proceeds by bounding the Rademacher complexity of the robust loss class
of the convex hull of $\H$, 
which implies the sample complexity \eqref{eq:attias} can also be achieved by $\RERM_\H$ (they propose an alternative, improper, learning rule for computational reasons).
But when $\max \abs{\U(x)}\leq k$, the second expression in our \eqref{eqn:ub-intermediate-mk} would be at most $O(\vc(\H) \log( m k ))$.  Thus, following the compression argument as in the proof of Theorem~\ref{thm:vc-dim} would yield the following sample complexity for our improper rule:
{\vskip -4mm}\begin{equation*}
\SC_{\RE}(\eps,\delta;\H,\U) = 
O\!\left( 
\tfrac{\vc(\H) \log(k)}{\eps} \log\!\left( \tfrac{\vc(\H) \log(k)}{\eps} \right) 
+ \tfrac{\vc(\H)}{\eps} \log^{2}\!\left( \tfrac{\vc(\H)}{\eps} \right)
+ \tfrac{1}{\eps} \log\!\left(\tfrac{1}{\delta}\right) \right),
\end{equation*}
{\vskip -1mm}\noindent and hence by Theorem~\ref{thm:agnostic-reduction}:
{\vskip -3mm}\begin{equation*}
\SC_{\AG}(\eps,\delta;\H,\U) = 
O\!\left( \tfrac{\vc(\H)\log(k)}{\eps^{2}}{\rm polylog}\!\left(\tfrac{\vc(\H) \log(k)}{\eps}\right) + \tfrac{1}{\eps^{2}}\log\!\left(\tfrac{1}{\delta}\right) \right).
\end{equation*}
{\vskip -1mm}In particular, our approach reduces the dependence on $k$ from $k\log(k)$ in \eqref{eq:attias} as obtained by \citet*{attias2018improved}, to $\log(k)(\log\log(k))^{3}$.  To do so, our approach {\em does} rely on improper learning, and our arguments are not valid for $\RERM_\H$.
We do not know whether improperness is {\em required} to obtain this improvement, or whether in this case a $\polylog k$ dependence is possible even with $\RERM$ or some other proper learning rule.  It follows from the construction of our negative result for proper learning in Theorem~\ref{thm:pac_properfail}, that at least a $\log(k)$ factor is sometimes necessary for proper learning (regardless of the VC dimension), whereas our Corollary~\ref{cor:vc-dim-agnostic-ub} 
implies that improper learning can achieve a sample complexity that is entirely independent of $k$ (albeit with a worse dependence on the VC dimension).

\section{Necessary and Sufficient conditions for Robust Learnability}
\label{sec:adjacent}
{\vskip -2mm}In the previous section, we saw that having finite VC dimension is {\em sufficient} for robust learnability.  But a simple construction shows that it is not {\em necessary}: consider an infinite domain $\X$, the hypothesis class of all possible predictors $\H = \{-,+\}^{\X}$, and an all-powerful adversary specified by $\U(x)=\X$. In this case, the hypothesis minimizing the population robust risk $\Risk_{\U}(h;\D)$ would always be the all-positive or the all-negative hypothesis, and so these are the only two hypothesis we should compete with.  And so, even though $\vc(\H)=\infty$, a single example suffices to inform the learner of whether to produce the all-positive or all-negative function. 


Can we then have a tight characterization of robust learnability?  Is there a weaker notion that is both necessary and sufficient for learning? A simple complexity measure one might consider is the maximum number of points $x_1,\dots,x_m$ such that the entire perturbation sets $\U(x_1),\dots,\U(x_m)$ are shattered by $\H$.  That is, such that $\forall y_{1},\ldots,y_{m}\in \{+1,-1\}, \exists h \in \H, \forall i \forall x'\in \U(x_i), \, h(x')=y_i$. We denote this as $\dim_{\U\times}(\H)$. When $\U(x)$ are balls around $x$, which is the typical case in metric-based robustness, this can be thought of shattering with a margin in input space.  Indeed, for linear predictors and when $\U(x)=\{x'|\norm{x-x'}_2\leq\gamma\}$ is a Euclidean ball around $x$, $\dim_{\U\times}(\H)$ exactly agrees with the fat shattering dimension at scale $\gamma$ (or the $VC_\gamma$ dimension).

While it is  fairly obvious that $\dim_{\U\times}(\H)$ provides a lower bound on the sample complexity of robust learning, and thus 
its finiteness is {\em necessary} for learning, we construct an example in Appendix~\ref{appendix-necessary} showing that it is 
not {\em sufficient}.
Specifically,
there are classes where \emph{no} points can be shattered in this way, and yet the classes are not robustly learnable.  Formally,  

\vspace{-1mm}\begin{proposition}
\label{prop:dims}
There exist $\X$, $\H$, $\U$ such that $\dim_{\U\times}\!(\H) = 0$ but $\SC_{\RE}(\epsilon,\delta;\H,\U)=\infty$.
\end{proposition}

{\vskip -1mm}We now attempt to refine the above measure, and introduce a weaker notion of robust shattering that that can still be used to lower bound the sample complexity for robust learnability. Given an adversary $\U$ and a hypothesis class $\H$, consider the following notion of $\U$-robust shattering,
\begin{definition}[Robust Shattering Dimension]
\label{def:robust_shatter}
A sequence $x_{1},\ldots,x_{m} \!\in \X$ is 
said to be \emph{$\U$-robustly shattered} by $\H$ if 
$\exists z_{1}^{+},z_{1}^{-},\ldots,z_{m}^{+},z_{m}^{-} \in \X$ 
with $x_i \in \U(z^{+}_i)\cap\U(z^{-}_i)$ $\forall i\in[m]$, 
and $\forall y_{1},\ldots,y_{m} \in \{-,+\}$, 
$\exists h \in \H$ with 
$h(z^{\prime}) = y_{i}$, 
$\forall z^{\prime} \in \U(z_{i}^{y_{i}})$,
$\forall i \in [m]$.
The \emph{$\U$-robust shattering dimension} $\dim_{\U}(\H)$ 
is defined as the largest $m$ for which there exist 
$m$ points $\U$-robustly shattered by $\H$.
\end{definition}

We have that $\dim_{\U\times}(\H) \leq \dim_{\U}(\H) \leq \vc(\H)$, 
where the first inequality follows since disjoint robust shattering is a special case of 
robust shattering with $z_{i}^{y} = x_{i}$, and so $\dim_{\U}(\H)$ is a plausible candidate for a necessary and sufficient dimension of robust learnability.  The following theorem (proof provided in appendix \ref{appendix-necessary}) establishes that the sample complexity of robust learnability is indeed lower bounded by the \emph{$\U$-robust shattering dimension} $\dim_{\U}(\H)$,

\vspace{-1mm}\begin{theorem}
\label{thm:robustdim-lower}
For any $\X$, $\H$, and $\U$,
\\$\SC_{\RE}(\epsilon,\delta;\H,\U) = \Omega\!\left( \frac{\dim_{\U}(\H)}{\eps} + \frac{1}{\eps} \log\!\left( \frac{1}{\delta} \right) \right)$ and $\SC_{\AG}(\epsilon,\delta;\H,\U) = \Omega\!\left( \frac{\dim_{\U}(\H)}{\eps^2} + \frac{1}{\eps^2} \log\!\left( \frac{1}{\delta} \right) \right)$.
\end{theorem}

\noindent Based on Corollary~\ref{cor:vc-dim} and Theorem~\ref{thm:robustdim-lower}, for any adversary $\U$ and any hypothesis class $\H$, we have
{\vskip -6mm}\begin{equation}
   \Omega\!\left( \tfrac{\dim_{\U}(\H)}{\eps} + \tfrac{1}{\eps} \log\!\left( \tfrac{1}{\delta} \right) \right) \leq \SC_{\RE}(\eps,\delta;\H,\U) \leq  2^{O(\vc(\H))} \tfrac{1}{\eps} \log\!\left(\tfrac{1}{\eps}\right) + O\!\left(\tfrac{1}{\eps}\log\!\left(\tfrac{1}{\delta}\right) \right).
\end{equation}
{\vskip -2mm}\noindent That is, the VC dimension is sufficient, and the robust shattering dimension is necessary for robust learnability.  As discussed at the beginning of the Section, we know the VC dimension is {\em not} necessary and there can be an arbitrary large, even infinite, gap in the second inequality.  We do not know whether the robust shattering dimension is also sufficient for learning, or whether there can also be a big gap in the first inequality.  Establishing a complexity measure that characterizes robust learnability thus remains an open question.

\section{Discussion and Future Directions}
\label{sec:questions}
{\vskip -2mm}Perhaps one of the most interesting takeaways from this work is that we should start considering \emph{improper} learning algorithms for adversarially robust learning. Even though our improper learning rule might not be practical, our results suggest to consider departing from robust empirical risk minimization and M-estimation (as in almost all published work), and considering \emph{improper} learning rules such as bagging or other ensemble methods.

Although we settled the question of robust learnability of VC classes, there remains a large gap in the question of what is the optimal sample complexity for robust learning. 
Can the exponential dependence on $\vc(\H)$ in Corollaries \ref{cor:vc-dim} and \ref{cor:vc-dim-agnostic-ub} be improved to a linear dependence?  Perhaps this is possible with a new analysis of our learning rule or a different \emph{improper} learning rule.  Since our learning rule and analysis stem from recent progress on compression schemes for VC classes \citep{moran:16}, it is certainly possible that further progress on the celebrated open problem regarding the existence of $\vc(\H)$ compression schemes \citep{floyd:95,warmuth:03} could also assist in progress on adversarially robust learning.

Our results demonstrate that there {\em exist} hypothesis classes with large gaps between what can be done with proper vs.~improper robust learning.  This means that when studying a particular class, such as classes corresponding to neural networks, one should consider the possibility that there {\em might} be such a gap and that improper learning {\em might} be necessary.  It remains open to establish whether such gaps actually exist for specific interesting neural net classes (e.g., functions representable by a specific architecture, possibly with a bounded weight norm).

\removed{ Another interesting direction for future work is studying adversarially robust learning in the multiclass classification setting. \citep{yin2018rademacher} proposed proper learning rules for multiclass linear classifiers and neural networks with robust generalization guarantees using margin analysis and Rademacher complexity. \natinote{I am not sure what we are saying here and how Yin's paper is relevant, what it is asking, or what it is answering.} }

Throughout the paper we ignored computational considerations.
Our learning rule can be viewed as an algorithm with black-box access to $\RERM_\H$, 
but making order $m^{\vc(\H)}$ such calls, 
and additionally requiring order $m^{\vc(\H) \vc^{*}(\H)}$ time and space 
to represent and update the distributions used by the boosting algorithm.  
Without significantly increasing the sample complexity, 
is it possible to robustly learn with an algorithm making only a polynomial (in $\vc(\H),\vc^{*}(\H),m$) 
number of calls to $\RERM_\H$ or even $\ERM_\H$, plus polynomial additional time and space?
What about ${\rm poly}(\vc(\H),m)$?  This question becomes even more interesting if there is such an 
algorithm that also only requires sample size $m = {\rm poly}(\vc(\H),1/\eps,\log(1/\delta))$,
rather than the $m = {\rm poly}(\vc(\H),\vc^{*}(\H),1/\eps,\log(1/\delta))$ sufficient for our algorithm.
Would another type of oracle be useful?  For example, can one devise efficient methods that rely on black-box access to $\ERM$ on the dual of the hypothesis class (i.e.~finding an example that is correct for the largest number of hypotheses in a given finite set of hypotheses)?   
More ambitiously, one may ask whether efficient PAC learnability implies efficient robust PAC learnability, 
roughly translating to asking whether access to {\em any} (non-robust) learning rule is sufficient for efficient robust learning. 

As a final remark, we note that our results easily extend to the multiclass setting ($|\Y| > 2$). 
In that case, 
by essentially the same algorithms and proofs, 
Theorems~\ref{thm:vc-dim} and \ref{thm:vc-dim-agnostic-ub} (and Corollaries~\ref{cor:vc-dim} and \ref{cor:vc-dim-agnostic-ub}) 
will hold with $\vc(\H)$ replaced by the \emph{graph dimension} \citep{natarajan:89,ben-david:95,daniely:15}.
The lower bound in Theorem~\ref{thm:robustdim-lower} 
also holds, 
by 
the same arguments,
but with 
$\dim_{\U}(\H)$ generalized  
analogous to the \emph{Natarajan dimension} \citep*{natarajan:89}: that is, 
in the definition of robust shattering, after ``and'', 
we now require $\forall i \exists y_{i,-},y_{i,+} \in \Y$ 
s.t.\ $\forall b_{1},\ldots,b_{m} \in \{-,+\}$, 
$\exists h \in \H$ with $h(z^{\prime})=y_{i,b_{i}}$, 
$\forall z^{\prime} \in \U(z_{i}^{b_{i}})$, $\forall i$. 
We leave as an open question whether one can also express an upper bound controlled by this quantity.

\vspace{-1mm}\subsection*{Acknowledgments}{

{\vskip -1mm}This work is partially funded by NSF-BSF award 1718970 and NSF award 1764032.

}

\bibliography{learning}

\appendix

\section{Auxilliary Proofs Related to Proper Robust Learnability}
\label{appendix-proper}

\begin{proof}[of Lemma~\ref{lemma:properfail_realizable}]
This proof follows standard lower bound techniques that use the probabilistic method \citep[Chapter~5]{shalev2014understanding}. Let $m\in \nats$. Construct $\H_0$ as before, according to Lemma $\ref{lemma:vcblowup}$, on $3m$ points $x_1,\dots,x_{3m}$. By construction, we know that $\L^{\U}_{\H_0}$ shatters the set $C=\{(x_1,+1),\dots,(x_{3m},+1)\}$. We will only keep a subset $\H$ of $\H_0$ that includes classifiers that are robustly correct only on subsets of size $2m$, i.e. $\H = \{ h_b \in \H_0: \sum_{i=1}^{3m} b_i = m\}$. Let $\A: (\X \times \Y)^* \mapsto \H$ be an arbitrary proper learning rule. The main idea here is to construct a family of distributions that are supported only on $2m$ points of $C$, which would force rule $\A$ to choose which points it can afford to be not correctly robust on. If rule $\A$ observes only $m$ points, it can't do anything better than guessing which of the remaining $2m$ points of $C$ are actually included in the support of the distribution.

Consider a family of distributions $\D_1, \dots, \D_T$ where $T = {3m\choose 2m}$, each distribution $\D_i$ is uniform over only $2m$ points in $C$. For every distribution $\D_i$, by construction of $\H$, there exists a classifier $h^*\in \H$ such that $\Risk_{\U}(h^*;\D_i) = 0$. This satisfies the first requirement. For the second requirement, we will use the probabilistic method to show that there exists a distribution $\D_i$ such that 
$\E_{S\sim \D_{i}^{m}} \Big[\Risk_{\U}(A(S);\D_i)\Big]\geq 1/4$, and finish the proof using a variant of Markov's inequality.  

Pick an arbitrary sequence $S \in C^m$. Consider a uniform weighting over the distributions $\D_1,\dots,\D_T$. Denote by $E_S$ the event that $S\subset \supp(\D_i)$ for a distribution $\D_i$ that is picked uniformly at random. We will lower bound the expected robust loss of the classifier that rule $\A$ outputs, namely $\A(S) \in \H$, given the event $E_S$,

\begin{equation}
\label{eqn:a}
\E_{\D_i} \!\left[\Risk_{\U}(\A(S);\D_i) | E_S \right] = \E_{\D_i} \!\left[ \E_{(x,y)\sim \D_i} \!\left[\sup_{z\in\U(x)} \ind[\A(S)(z)\neq y]\right] \middle| E_S \right].
\end{equation}
We can lower bound the robust loss of the classifier $\A(S)$ by conditioning on the event that  $(x,y)\notin S$ denoted $E_{(x,y)\notin S}$,
\begin{equation*}
\underset{(x,y)\sim \D_i}{\E} \Big[\sup_{z\in\U(x)} \ind[\A(S)(z)\neq y]\Big] \geq \underset{(x,y)\sim \D_i}{\P}[E_{(x,y)\notin S}] \underset{(x,y)\sim \D_i}{\E}[\sup_{z\in\U(x)} \ind[\A(S)(z)\neq y]| E_{(x,y)\notin S}].
\end{equation*}
{\vskip -2mm}Since $|S|=m$, and $\D_i$ is uniform over its support of size $2m$, we have $\P_{(x,y)\sim \D_i}[E_{(x,y)\notin S}]\geq 1/2$. This allows us to get a lower bound on \eqref{eqn:a},
\begin{equation}
\label{eqn:b}
    \E_{\D_i} \Big[\Risk_{\U}(\A(S);\D_i) | E_S \Big] \geq \frac{1}{2} \E_{\D_i} \Bigg[\underset{(x,y)\sim \D_i}{\E}\Big[\sup_{z\in\U(x)} \ind[\A(S)(z)\neq y]\Big| E_{(x,y)\notin S}\Big]  \Bigg| E_S \Bigg].
\end{equation}
{\vskip -2mm}Since $\A(S) \in \H$, by construction of $\H$, we know that there are at least $m$ points in $C$ where $\A(S)$ is not robustly correct. We can unroll the expectation over $\D_i$ as follows
\begin{align*}
        &\E_{\D_i} \Bigg[\underset{(x,y)\sim \D_i}{\E}\Big[\sup_{z\in\U(x)} \ind[\A(S)(z)\neq y]| E_{(x,y)\notin S}\Big]  \Big| E_S \Bigg]
        \\&\geq \frac{1}{m} \!\!\!\sum_{(x,y)\notin S} \!\!\!\E_{\D_i}[\ind_{(x,y)\in \supp(\D_i )} | E_S] \!\!\sup_{z\in\U(x)}\!\! \ind[\A(S)(z)\neq y]
        \geq \frac{1}{m} \!\!\!\sum_{(x,y)\notin S} \!\!\frac{1}{2}\! \sup_{z\in\U(x)} \ind[\A(S)(z)\neq y] \geq \frac{1}{2}.
\end{align*}
{\vskip -2mm}Thus, it follows by \eqref{eqn:b} that $\E_{\D_i} \Big[\Risk_{\U}(\A(S);\D_i) | E_S \Big] \geq \frac{1}{4}$.
Now, by law of total expectation, 
\begin{center}
    {\vskip -1mm}$\E_{\D_i} \!\left[ \E_{S\sim \D_{i}^{m}} \!\left[\Risk_{\U}(\A(S);\D_i)\right] \right] = \E_{S\sim \D_{i}^{m}} \!\left[ \E_{\D_i} \!\left[\Risk_{\U}(\A(S);\D_i) \middle| E_S\right] \right]\geq \frac{1}{4}$.
\end{center}
{\vskip -1mm}Since the expectation over $\D_1,\dots,\D_T$ is at least $1/4$, this implies that there exists a distribution $\D_i$ such that $\E_{S\sim \D_{i}^{m}} \Big[\Risk_{\U}(\A(S);\D_i)\Big]\geq 1/4$. Using a variant of Markov's inequality, for any random variable $Z$ taking values in $[0,1]$, and any $a\in(0,1)$, we have $\P[Z>1-a] \geq \frac{\E[Z] - (1-a)}{a}$. For $Z=\Risk_{\U}(\A(S);\D_i)$ and $a=7/8$, we get $\P_{S\sim \D_{i}^{m}} \Big[\Risk_{\U}(\A(S);\D_i) > \frac{1}{8}\Big]\geq \frac{1/4 - 1/8}{7/8} = \frac{1}{7}$.
\end{proof}

\section{Auxilliary Proofs Related to Realizable Robust Learnability}
\label{appendix-realizable}

The following lemma extends the classic compression-based generalization guarantees 
from the $0$-$1$ loss to also hold for the robust loss.  It is used in the proof of Theorem~\ref{thm:vc-dim}.
Generally, it is also possible to extend other generalization guarantees for compression schemes 
to the robust loss, such as improved bounds for permutation-invariant compression schemes, 
or convergence guarantees for the agnostic case (as discussed in Section~\ref{subsec:agnostic}).

\begin{lemma}
\label{lem:robust-compression}
For any $k \in \nats$ and fixed function $\phi : (\X \times \Y)^{k} \to \Y^{\X}$, 
for any distribution $P$ over $\X \times \Y$ and any $m \in \nats$, 
for $S = \{(x_{1},y_{1}),\ldots,(x_{m},y_{m})\}$ iid $P$-distributed random variables,
with probability at least $1-\delta$, 
if $\exists i_{1},\ldots,i_{k} \in \{1,\ldots,m\}$ 
s.t.\ $\hat{R}_{\U}(\phi((x_{i_{1}},y_{i_{1}}),\ldots,(x_{i_{k}},y_{i_{k}}));S) = 0$, 
then 
\begin{equation*}
R_{\U}(\phi((x_{i_{1}},y_{i_{1}}),\ldots,(x_{i_{k}},y_{i_{k}}));P) \leq \frac{1}{m-k} (k\ln(m) + \ln(1/\delta)).
\end{equation*}
\end{lemma}
\begin{proof}
For completeness, we include a brief proof,
which merely notes that the classic argument of \citep*{littlestone:86,floyd:95} 
establishing generalization guarantees for sample compression schemes under the $0$-$1$ loss remains 
valid under the robust loss.  

For any indices $i_{1},\ldots,i_{k} \in \{1,\ldots,m\}$, 
\begin{align*}
& \P\!\left( \hat{\Risk}_{\U}(\phi(\{(x_{i_{j}},y_{i_{j}})\}_{j=1}^{k}); S) = 0 \text{ and }
\Risk_{\U}(\phi(\{(x_{i_{j}},y_{i_{j}})\}_{j=1}^{k});P) > \eps \right)
\\ & \leq \E\!\left[ \P\!\left( \hat{\Risk}_{\U}(\phi(\{(x_{i_{j}},y_{i_{j}})\}_{j=1}^{k}); S \setminus \{(x_{i_{j}},y_{i_{j}})\}_{j=1}^{k}) = 0 \middle| \{(x_{i_{j}},y_{i_{j}})\}_{j=1}^{k} \right) \times \right.
\\ & \hspace{6cm}\left. \phantom{\P\!\left( \hat{\Risk}_{\U} \middle| \right)}\ind\!\left[ \Risk_{\U}(\phi(\{(x_{i_{j}},y_{i_{j}})\}_{j=1}^{k});P) > \eps \right] \right]
\\ & < (1-\eps)^{m-k},
\end{align*}
and a union bound over all $m^{k}$ possible choices of $i_{1},\ldots,i_{k}$ implies 
a probability at most $m^{k} (1-\eps)^{m-k} \leq m^{k} e^{-\eps (m-k)}$ that there exist $i_{1},\ldots,i_{k}$ with 
$\Risk_{\U}(\phi(\{(x_{i_{j}},y_{i_{j}})\}_{j=1}^{k});\PXY) > \eps$ and yet $\hat{\Risk}_{\U}(\phi(\{(x_{i_{j}},y_{i_{j}})\}_{j=1}^{k}); S) = 0$.
This is at most $\delta$ for a choice of $\eps = \frac{1}{m-k} (k\ln(m) + \ln(1/\delta))$.
\end{proof}

\section{Proof of Agnostic Robust Learnability}
\label{appendix-agnostic}

\begin{proof}[of Theorem~\ref{thm:agnostic-reduction}]
The argument follows closely a proof of an analogous result by \citet*{david:16} for non-robust learning.
Denote by $\alg$ the optimal 
realizable-case learner achieving sample complexity 
$\SC_{\RE}(1/3,1/3;\H,\U)$, and 
denote $\Mre = \SC_{\RE}(1/3,1/3;\H,\U)$,
as above.

Then, in the agnostic case, 
given a data set $S \sim \D^{m}$, 
we first do robust-ERM to find a maximal-size 
subsequence $S^{\prime}$ of the data where the robust 
loss can be zero: that is, 
$\inf_{h \in \H} \hat{\Risk}_{\U}(h;S^{\prime}) = 0$.
Then for any distribution $D$ over $S^{\prime}$, 
there exists a sequence $S_{D} \in (S^{\prime})^{\Mre}$ 
such that $h_{D} := \alg(S_{D})$ has 
$\Risk_{\U}(h_{D};D) \leq 1/3$; 
this follows since, by definition of 
$\SC_{\RE}(1/3,1/3;\H,\U)$, there is a $1/3$ 
chance that $\hat{S}$ a random draw from $D^{\Mre}$ 
yields $\Risk_{\U}(\alg(\hat{S});D) \leq 1/3$, 
so at least one such $S_{D}$ exists.
We use this to define a weak robust-learner for 
distributions $D$ over $S^{\prime}$: 
i.e., for any $D$, the weak learner chooses $h_{D}$ 
as its weak hypothesis.

Now we run the $\alpha$-Boost boosting algorithm 
\citep*[][Section 6.4.2]{schapire:12} 
on data set $S^{\prime}$, 
but using the robust loss rather than $0$-$1$ loss.  
That is, we start with $D_{1}$ uniform on $S^{\prime}$.\footnote{We ignore the possibility of repeats; 
for our purposes we can just remove any repeats from $S^{\prime}$ 
before this boosting step.}
Then for each round $t$, we get $h_{D_{t}}$ 
as a weak robust classifier with respect to $D_{t}$, 
and for each $(x,y) \in S^{\prime}$ we define 
a distribution $\D_{t+1}$ over $S^{\prime}$ satisfying 
\begin{equation*}
D_{t+1}(\{(x,y)\}) \propto 
D_{t}(\{(x,y)\}) \exp\!\left\{-2\alpha \ind[ \forall x^{\prime} \in \U(x), h_{D_{t}}(x^{\prime})=y ] \right\},
\end{equation*}
where $\alpha$ is a parameter we can set.
Following the argument from \citet*[][Section 6.4.2]{schapire:12}, after $T$ rounds we are guaranteed 
\begin{equation*}
\min_{(x,y) \in S^{\prime}} \frac{1}{T} \sum_{t=1}^{T} \ind[ \forall x^{\prime} \in \U(x), h_{D_{t}}(x^{\prime})=y ]
\geq \frac{2}{3} - \frac{2}{3}\alpha - \frac{\ln(|S^{\prime}|)}{2\alpha T},
\end{equation*}
so we will plan on running until round 
$T = 1 + 48 \ln(|S^{\prime}|)$ 
with value 
$\alpha = 1/8$ 
to guarantee
\begin{equation*}
\min_{(x,y) \in S^{\prime}} \frac{1}{T} \sum_{t=1}^{T} \ind[ \forall x^{\prime} \in \U(x), h_{D_{t}}(x^{\prime})=y ]
> \frac{1}{2},
\end{equation*}
so that the classifier 
$\hat{h}(x) := \ind\!\left[ \frac{1}{T} \sum_{t=1}^{T} h_{D_{t}}(x) \geq \frac{1}{2} \right]$ 
has $\hat{\Risk}_{\U}(\hat{h};S^{\prime}) = 0$.

Furthermore, note that, since each $h_{D_{t}}$
is given by $\alg(S_{D_{t}})$, where $S_{D_{t}}$ is an 
$\Mre$-tuple of points in $S^{\prime}$, 
the classifier $\hat{h}$ is specified by an 
ordered sequence of $\Mre T$ points from $S$.
Altogether, $\hat{h}$ is a function specified 
by an ordered sequence of $\Mre T$ points 
from $S$, and which has 
\begin{equation*}
\hat{\Risk}_{\U}(\hat{h};S) \leq \min_{h \in \H} \hat{\Risk}_{\U}(h;S).
\end{equation*}
Similarly to the realizable case (see the proof of Lemma~\ref{lem:robust-compression}), 
uniform convergence guarantees for sample compression 
schemes \citep*[see][]{graepel:05} remain valid for the robust loss, 
by essentially the same argument; the essential argument is the 
same as in the proof of Lemma~\ref{lem:robust-compression} except 
using Hoeffding's inequality to get concentration of the empirical 
robust risks for each fixed index sequence, and then a union bound over 
the possible index sequnces as before.
We omit the details for brevity. 
In particular, denoting $T_{m} = 1 + 48 \ln(m)$, 
for $m > \Mre T_{m}$, with probability at least $1-\delta/2$, 
\begin{equation*}
\Risk_{\U}(\hat{h};\D)
\leq \hat{\Risk}_{\U}(\hat{h};S) 
+ \sqrt{\frac{\Mre T_{m} \ln(m) + \ln(2/\delta)}{2m - 2\Mre T_{m}}}.
\end{equation*}

Let $h^{*} = \argmin_{h \in \H} \Risk_{\U}(h;\D)$ 
(supposing the min is realized, for simplicity; else 
we could take an $h^{*}$ with very-nearly minimal risk). 
By Hoeffding's inequality, with probability at least 
$1-\delta/2$, 
\begin{equation*}
\hat{\Risk}_{\U}(h^{*};S) 
\leq \Risk_{\U}(h^{*};\D) 
+ \sqrt{\frac{\ln(2/\delta)}{2m}}.
\end{equation*}

By the union bound, if $m \geq 2 \Mre T_{m}$, with probability at least $1-\delta$, 
\begin{align*}
\Risk_{\U}(\hat{h};\D) 
& \leq \min_{h \in \H} \hat{\Risk}_{\U}(h;S) 
+ \sqrt{\frac{\Mre T_{m} \ln(m) + \ln(2/\delta)}{m}}
\\ & \leq \hat{\Risk}_{\U}(h^{*};S) 
+ \sqrt{\frac{\Mre T_{m} \ln(m) + \ln(2/\delta)}{m}}
\\ & \leq \Risk_{\U}(h^{*};\D) 
+ 2\sqrt{\frac{\Mre T_{m} \ln(m) + \ln(2/\delta)}{m}}.
\end{align*}
Since 
$T_{m} = O( \log(m) )$, 
the above is at most $\eps$ 
for an appropriate choice of sample size 
$m = O\!\left( \frac{\Mre}{\eps^{2}} \log^{2}\!\left(\frac{\Mre}{\eps}\right) + \frac{1}{\eps^{2}}\log\!\left(\frac{1}{\delta}\right) \right)$.
\end{proof}

\section{Auxilliary Proofs Related to Necessary Conditions for Robust Learnability}
\label{appendix-necessary}

\begin{proof}[of Proposition~\ref{prop:dims}]
Let $\X=\reals^d$ equipped with a metric $\dist$, and $\U: \X \mapsto 2^\X$ such that $\U(x)=\{z\in \X: \dist{(x,z)} \leq \gamma\}$ for all $x\in\X$ for some $\gamma > 0$. Consider two infinite sequences of points $(x_m)_m\in\nats$ and $(z_m)_m\in\nats$ such that for any $i\neq j$, $\U(x_i)\cap\U(x_j)=\emptyset$, $\U(x_i)\cap \U(z_j)=\emptyset$, $\U(x_j)\cap \U(z_i)=\emptyset$, but $\U(x_i)\cap \U(z_i)={u_i}$. In other words, we want the $\gamma$-balls of pairs with different indices to be mutually disjoint, and the $\gamma$-balls for a pair with the same index to intersect at a single point (this is possible because we are considering closed balls).

Next, we proceed with the construction of $\H$. For each bit string $b\in \{0,1\}^{\nats}$, we will define a predictor $h_b:\X \mapsto \Y$ just on the $\gamma$-balls of the points $x_1,z_1,x_2,z_2,\dots$ (it labels the rest of the $\X$ space with $+1$). Foreach $i\in \nats$, if $b_i = 0$, set
\begin{equation*}
    h_b\Big(\U(x_i)\Big) = +1 \quad \text{and} \quad h_b\Big(\U(z_i) \setminus \U(x_i)\Big)=-1
\end{equation*}
and if $b_i=1$, set
\begin{equation*}
    h_b\Big(\U(x_i) \setminus \U(z_i) \Big) = -1 \quad \text{and} \quad h_b\Big(\U(z_i)\Big)=+1
\end{equation*}

Let $\H = \{h_b: b\in\{0,1\}^{\nats}\}$. Notice that $\dim_{\U\times}(\H)=0$, because there is no single $\gamma$-ball that is labeled in both ways ($+1$ and $-1$). By construction of $\H$, all classifiers $h_b \in \H$ behave the same way on all points in $\X$, except at points in the intersections $u_1,u_2,\dots$ which get shattered. However, the $\U$-robust shattering dimension (see definition \ref{def:robust_shatter}) is infinite in this construction ($\dim_{\U}(\H)=\infty$), which by Theorem \ref{thm:robustdim-lower} (see below) implies that $\SC_{\RE}(\epsilon,\delta;\H,\U)=\infty$. 
\end{proof}

\begin{proof}[Sketch of Theorem \ref{thm:robustdim-lower}]
We first start with the realizable case. The proof follows a standard argument from \citep[Chapter~3]{mohri2018foundations}. Let $\d = \dim_{\U}(\H)$, and fix $x_{1},\ldots,x_{\d}$ 
a sequence $\U$-robustly shattered by $\H$, 
and let $z_{1}^{+},z_{1}^{-},\ldots,z_{\d}^{+},z_{\d}^{-} \in \X$ 
be as in definition \ref{def:robust_shatter}; in particular, note that any 
$y,y^{\prime}$ and any $i \neq j$ necessarily have  
$z_{i}^{y} \neq z_{j}^{y^{\prime}}$.  
For each $\y = (y_{1},\ldots,y_{\d}) \in \{+1,-1\}^{\d}$, 
let $h^{\y} \in \H$ 
be such that $\forall i \in [m]$, 
$\forall z^{\prime} \in \U(z_{i}^{y_{i}})$,  
$h^{\y}(z^{\prime}) = y_{i}$. Let $\D$ be a distribution over $\{1,2,\dots,\d\}$ such that $\P_{i\sim \D}[i=1]=1-8\eps$ and $\P_{i\sim \D}[i=1]=8\eps/(\d-1)$ for $2 \leq i\leq d$. Now choose $\y \sim {\rm Uniform}(\{+1,-1\}^{\d})$, and let $\D_{\y}$ be the induced distribution over $\X\times\Y$ such that
\begin{equation*}
    \P_{(x,y)\sim\D_{\y}}\left[(x,y)=(z_1^{y_1},y_1)\right]=1-8\epsilon \text{ and } \P_{(x,y)\sim\D_{\y}}\left[(x,y)=(z_i^{y_i},y_i)\right]=8\epsilon/(\d-1)
\end{equation*}
for $2 \leq i\leq \d$.


Note that by construction we have $\Risk_{\U}(h^{\y};\D) = 0$. Now, consider an arbitrary learning rule $\A:(\X \times \Y)^*\mapsto \Y^\X$. We will assume that $\A$ always gets the prediction of $z_1^{y_1}$ correct.
Let $I = \{2,\dots,d\}$ and let $\S$ be the set of all sequences of size $m$ containing at most $(\d-1)/2$ elements from $I$. Fix an arbitrary sequence $S\in \S$. Denote by $S_\y=((z_i^{y_i}, y_i): i\in\S)$ the sequence of examples induced by the indices sequence $S$. Then,
\begin{align*}
    \E_{\y} \left[\Risk_{\U}(\A(S_\y);\D_{\y}) \right] &\geq \E_{\y} \left[ \sum_{i\notin S} \P_{\D_\y}(z_i^{y_i})\sup_{z'\in\U(z_i^{y_i})} \ind[\A(S_\y)(z')\neq y_i]\right]\\
    &\geq \frac{\d-1}{2} \times \frac{8\epsilon}{\d-1} \times \E_{\y} \left[\sup_{z'\in\U(z)} \ind[\A(S_\y)(z')\neq y]\right]\\
    &= \frac{\d-1}{2}\times \frac{8\epsilon}{\d-1} \times \frac{1}{2}\\
    &= 2\epsilon
\end{align*}

Since the inequality above holds for any sequence $S\in\S$, it follows that 
\begin{equation*}
    \E_{S\sim\D^m}[\E_{\y} \left[ \Risk_{\U}(\A(S_\y);\D_{\y}) \ind_{S\in\S} \right]]=\E_{\y}[\E_{S\sim\D^m} \left[ \Risk_{\U}(\A(S_\y);\D_{\y})|E_{S\in\S} \right]]\geq 2\epsilon
\end{equation*}
Which implies that there exists $\y_0$ such that $\E_{S\sim\D^m} \left[ \Risk_{\U}(\A(S_{\y_0});\D_{\y_0}) | E_{S\in\S} \right]\geq 2\epsilon$. Since $\P_{\D}[i \in I]\leq 8\epsilon$, the robust risk $\Risk_{\U}(\A(S_{\y_0}));\D_{\y}) \leq 8\epsilon$. Then, by law of total expectation, we have
\begin{align*}
    2\epsilon &\leq \E_{S\sim \D^m} \left[ \Risk_{\U}(\A(S_{\y_0});\D_{\y_0}) | E_{S\in\S}\right]\\ &\leq 8\epsilon \P_{S\sim\D^m} \left[ \Risk_{\U}(\A(S_{\y_0});\D_{\y}) \geq \epsilon|E_{S\in\S} \right] + \epsilon (1-\P_{S\sim\D^m} \left[ \Risk_{\U}(\A(S_{\y_0});\D_{\y_0}) \geq \epsilon \right|E_{S\in\S}])
\end{align*}

By collecting terms, we obtain that $\P_{S\sim\D^m} \left[ \Risk_{\U}(\A(S_{\y_0});\PXY) \geq \epsilon \right | E_{S\in\S}] \geq 1/7$. Then, by law of total probability, the probability over all sequences (not necessarily in $\S$) can be lower bounded,
\begin{equation*}
    \P_{S\sim\D^m} \left[ \Risk_{\U}(\A(S_{\y_0});\D_{\y_0}) \geq \epsilon \right] \geq \P[E_{S\in\S}]\P_{S\sim\D^m}\left[ \Risk_{\U}(\A(S_{\y_0});\D_{\y_0}) \geq \epsilon | E_{S\in\S}\right] \geq \frac{1}{7}\P[E_{S\in\S}] 
\end{equation*}

By a standard application of Chernoff bounds, for $\epsilon=\frac{\d-1}{32m}$ and $\delta\leq 1/100$, we get that $\P[E_{S\in\S}] \geq 7\delta$ and by the above this concludes that $\P_{S\sim\D^m} \left[ \Risk_{\U}(\A(S_{\y_0});\D_{\y_0}) \geq \epsilon \right]\geq \delta$. This establishes that
\begin{equation*}
    \SC_{\RE}(\epsilon,\delta;\H,\U) \geq \Omega\!\left( \frac{\d}{\epsilon} \right)
\end{equation*}

To finish the proof, we need to show that 
\begin{equation*}
    \SC_{\RE}(\epsilon,\delta;\H,\U) \geq \Omega\!\left( \frac{1}{\eps} \log\!\left( \frac{1}{\delta} \right) \right)
\end{equation*}

For this just consider a distribution $P_1$ with mass $1-\epsilon$ on $(z_1^{+},+1)$ and mass $\epsilon$ on $(z_2^{+},+1)$, and another distribution $P_2$ with mass $1-\epsilon$ on $(z_1^{+},+1)$ and mass $\epsilon$ on $(z_2^{-},-1)$. If $m\leq (1/2\epsilon)\ln(1/\delta)$, with probability at least $\delta$, we will only observe $m$ samples of $(z_1^{+},+1)$, and thus learning rule $\A$ will make a mistake on $x_2$ (which is in $\U(z_2^{+})\cap\U(z_2^{-})$) with probability at least $1/2$, therefore having error at least $\epsilon/2$. By combining both parts, we arrive at the theorem statement.

For the agnostic case, we briefly describe the construction. The remainder of the proof more or less follows a standard argument, for instance see \citet[~Chapter 5]{anthony:99}. Let $\d = \dim_{\U}(\H)$, and fix $x_{1},\ldots,x_{\d}$ 
a sequence $\U$-robustly shattered by $\H$, and let $z_{1}^{+},z_{1}^{-},\ldots,z_{\d}^{+},z_{\d}^{-} \in \X$ be as in definition \ref{def:robust_shatter}; in particular, note that any $y,y^{\prime}$ and any $i \neq j$ necessarily have $z_{i}^{y} \neq z_{j}^{y^{\prime}}$. For $b \in \{0,1\}^\d$, define distribution $\D_b$ as follows, for $i\in[\d]$:
\begin{itemize}
    \item If $b_i$ = 0, then set $\P_{\D_b}((z_i^{+},+1))=(1-\alpha)/(2\d)$ and $\P_{\D_b}((z_i^{-},-1))=(1+\alpha)/(2\d)$.
    \item If $b_i$ = 1, then set $\P_{\D_b}((z_i^{+},+1))=(1+\alpha)/(2\d)$ and $\P_{\D_b}((z_i^{-},-1))=(1-\alpha)/(2\d)$.
\end{itemize}
where $) < \alpha < 1$ is appropriately chosen based on $\eps$ and $\delta$.
\end{proof}

\end{document}